\documentclass[11pt]{article}

\usepackage[margin=1in]{geometry}
\usepackage{amsmath,amssymb,amsthm}
\usepackage[numbers]{natbib}

\newtheorem{theorem}{Theorem}

\title{On the Identifiability of User Adaptation in Co-Adaptive Neural Interfaces}
\author{Philip Waggoner \\ Stanford University \\ \texttt{pdw2119@stanford.edu}}
\date{}

\begin{document}

\maketitle

\section{Introductory Remarks}

Recently, Madduri et al.~\citep{madduri2026computational} presented an elegant experimental and theoretical framework for studying co-adaptive neural interfaces, in which a human user and a machine decoder jointly adapt in closed loop. The authors combine a game-theoretic formulation with an empirical design based on electromyographic (EMG) control. They show that decoder learning rates and regularization influence both performance and observed user behavior. This is a very important contribution. Buiulding on that contribution, we offer a clarification regarding what aspects of the human-machine system are identified by the experimental design. Of note, we show that while the Madduri et al. estimation procedure recovers joint closed-loop dynamics, it does not, in general, uniquely identify the user-side adaptation mechanism. This distinction, though subtle, is quite important for interpreting claims about predicting or shaping human behavior, in line with the broader goal of their study.

\section{Estimation Problem Setup}

We can start by stating the core empirical structure in compact form. Let $x_t \in \mathbb{R}^d$ denote task variables (e.g., target and cursor state), $u_t \in \mathbb{R}^m$ the user's motor output (EMG), and $y_t \in \mathbb{R}^k$ the controlled output (cursor velocity). The decoder $D_t \in \mathbb{R}^{k \times m}$ maps motor output to cursor motion:
\begin{equation}
  y_t = D_t u_t.
  \label{eq:decoder}
\end{equation}
The user implements an unknown control policy $U_t$, which we may write abstractly as
\begin{equation}
  u_t = U_t(x_t, \xi_t),
  \label{eq:user_policy}
\end{equation}
where $\xi_t$ captures internal noise or latent state. The empirical procedure estimates an ``encoder'' $\hat E_t \in \mathbb{R}^{m \times d}$ over a finite window $\mathcal{W}_t$ by linear regression:
\begin{equation}
  \hat E_t = \arg\min_E \sum_{s \in \mathcal{W}_t} \|u_s - E x_s\|^2.
  \label{eq:encoder_est}
\end{equation}
The decoder remains fixed within $\mathcal{W}_t$, but is updated across windows, producing a sequence $(\hat E_t, D_t)$. The analysis then interprets variation in $\hat E_t$ as evidence of user adaptation and relates these dynamics to a scalar game-theoretic model in which user and decoder minimize coupled objective functions reflecting task performance and effort costs \citep{madduri2026computational}.

\section{Identifiability Problem}

The key question is whether the sequence $\{\hat E_t\}$ identifies the latent user policy sequence $\{U_t\}$. We show that this is \textit{not generically} the case. Define the mapping
\begin{equation}
  \mathcal{M}: \{U_t, D_t, \mathcal{W}_t\} \mapsto \{\hat E_t\},
  \label{eq:mapping}
\end{equation}
which takes the latent policy, decoder trajectory, and windowing scheme to the sequence of estimated encoders. Identification of $U_t$ requires that $\mathcal{M}$ be injective in $U_t$ (up to equivalence). However, in closed-loop settings, $\mathcal{M}$ is typically many-to-one, as identifiability in these settings generally fails without additional excitation conditions (see, e.g., \citep{ljung1998system}). Intuitively, $\hat E_t$ depends not only on the mapping $U_t$, but also on the distribution of regressors $x_s$ induced by the evolving closed-loop system. So, changes in $D_t$ alter the state distribution, which in turn affects $\hat E_t$, even if $U_t$ is fixed. Thus, variation in $\hat E_t$ cannot be uniquely attributed to changes in $U_t$ without additional assumptions.

\section{A Non-Identifiability Result}

To make the identifiability problem clearer, we now formalize the identifiability limitation with a minimal result.

\begin{theorem}[Non-identifiability of user adaptation under closed-loop regression]
Consider a closed-loop system defined by Equations~\eqref{eq:decoder}-\eqref{eq:user_policy}, with encoder estimates $\hat E_t$ computed via Equation~\eqref{eq:encoder_est} over finite windows $\mathcal{W}_t$. Suppose that:
\begin{enumerate}
  \item The decoder sequence $\{D_t\}$ is time-varying,
  \item The task state $\{x_t\}$ evolves according to dynamics of the form
  \begin{equation}
    x_{t+1} = F_t(x_t, D_t, U_t) + \epsilon_t,
    \label{eq:general_dynamics}
  \end{equation}
  where $\epsilon_t$ is stochastic noise,
  \item The encoder estimates $\hat E_t$ are obtained from finite samples within each window $\mathcal{W}_t$.
\end{enumerate}
Then, in general, the mapping
\[
  \mathcal{M}: \{U_t\} \mapsto \{\hat E_t\}
\]
is not injective. That is, there exist distinct user policy sequences $\{U_t\}$ and $\{\tilde U_t\}$ such that
\[
  \mathcal{M}(\{U_t\}) = \mathcal{M}(\{\tilde U_t\}).
\]
\end{theorem}

\textit{Proof (by construction).}
We construct two observationally equivalent systems.

\textit{System A (fixed user).} Let the user policy be time-invariant:
\[
  u_t = E^\star x_t,
\]
with fixed matrix $E^\star$. The state evolves according to Equation~\eqref{eq:general_dynamics}, which induces a sequence of state distributions $\{x_t\}$ that depend on the decoder trajectory $\{D_t\}$.

\textit{System B (time-varying user).} Define a perturbed user policy
\[
  u_t = (E^\star + \Delta_t) x_t,
\]
where $\Delta_t$ is chosen such that, over each window $\mathcal{W}_t$, the joint empirical moments satisfy
\begin{equation}
  \sum_{s \in \mathcal{W}_t} u_s x_s^\top
  =
  E^\star \sum_{s \in \mathcal{W}_t} x_s x_s^\top.
  \label{eq:moment_matching}
\end{equation}
Such a sequence $\{\Delta_t\}$ exists whenever the regressor covariance matrices are not perfectly constrained (e.g., under finite samples or collinearity). Under condition~\eqref{eq:moment_matching}, the regression estimator in Equation~\eqref{eq:encoder_est} yields
\[
  \hat E_t = E^\star
\]
for both systems.

Thus, the distinct policy sequences $\{U_t\}$ and $\{\tilde U_t\}$ generate identical encoder estimates $\{\hat E_t\}$. Hence, $\mathcal{M}$ is not injective.

\section{Discussion}

The non-identifiability limitation detailed herein has direct implications for interpretation. The experimental results from \citet{madduri2026computational} demonstrate well that manipulations of the decoder trajectory $\{D_t\}$ affect joint system behavior and the sequence of estimated encoders $\{\hat E_t\}$. However, without identification of $U_t$, these results do not \textit{uniquely} determine the user's internal objective, learning dynamics, etc. In particular, the mapping from observed changes in $\hat E_t$ to a specific cost function is not one-to-one. Multiple user models, including fixed policies under shifting state distributions, delayed adaptation rules, or even alternative regularization structures, can reproduce the same empirical signatures. Thus, the framework identifies properties of the closed-loop system, but does not uniquely identify the underlying user adaptation dynamics. Claims about predicting or shaping user behavior, thus, should be interpreted at the level of the joint system, rather than the user in isolation.

\section{Concluding Remarks}

Of note, this identifiability limitation is not intrinsic to the framework, but rather to the current experimental design. Identification of user adaptation simply requires additional structure. One approach might be to introduce exogenous perturbations that generate persistent excitation independent of the closed-loop dynamics. For example, augmenting the decoder with randomized probe inputs $\tilde D_t = D_t + \delta_t$, where $\delta_t$ is independent noise, yields variation in $x_t$ that could separate user response from state distribution effects. 

Alternatively, a latent state formulation, e.g., $z_{t+1} = f(z_t, x_t, D_t), u_t = g(z_t, x_t),$ combined with constraints on $f$ and $g$, would permit identification under standard observability conditions. 

Finally, longer horizon experiments without frequent decoder reinitialization may allow separation of transient and persistent adaptation components. These extensions would enable the framework to move from identifying co-adaptive trajectories to identifying the underlying user adaptation mechanism.

\bibliographystyle{plainnat}
\bibliography{refs}

\end{document}